\documentclass[twoside]{article}

\usepackage[accepted]{aistats2024}
\usepackage{tikz}
\usepackage[utf8]{inputenc} 
\usepackage[T1]{fontenc}    
\usepackage[colorlinks = true,linkcolor = blue,urlcolor = blue, citecolor = blue]{hyperref}       
\usepackage{url}            
\usepackage{booktabs}       
\usepackage{amsfonts}       
\usepackage{amsmath,amssymb,amsthm}
\usepackage{nicefrac}       
\usepackage{microtype}      
\usepackage{xcolor}         
\usepackage{xspace}
\usepackage{natbib}
\usepackage{placeins}

\usepackage{todonotes}
\usepackage{cleveref}
\usepackage{comment}

\def\tmid{\!\mid\!}

\newcommand{\tw}{\operatorname{tw}}
\newcommand{\pw}{\operatorname{pw}}

\newcommand{\mvf}{\kappa}
\newcommand{\Oh}{\mathcal{O}}
\newcommand{\parvf}{maximum pairwise connectivity\xspace}
\newcommand{\Parvf}{MAXIMUM PAIRWISE CONNECTIVITY\xspace}

\newcommand{\bag}{\texttt{bag}}
\newcommand{\cmp}{\texttt{cmp}}

\newcommand{\connreply}{\textsf{Connected}\xspace}
\newcommand{\disconnreply}{\textsf{Disconnected}\xspace}

\newtheorem{theorem}{Theorem}[section]
\newtheorem{claim}[theorem]{Claim}

\newtheorem{lemma}[theorem]{Lemma}

\newcommand{\cqed}{\ensuremath{\lhd}}
\makeatletter
\newenvironment{claimproof}{\par
	\pushQED{\cqed}%
	\normalfont \topsep6\p@\@plus6\p@\relax
	\trivlist
	\item\relax
	{\itshape
		Proof of the claim\@addpunct{.}}\hspace\labelsep\ignorespaces
}{%
	\hfill\popQED\endtrivlist\@endpefalse
}
\makeatother

\begin{document}

\twocolumn[

\aistatstitle{Structural perspective on constraint-based learning of Markov networks}

\aistatsauthor{Tuukka Korhonen \And Fedor V. Fomin \And Pekka Parviainen}

\aistatsaddress{University of Bergen \And University of Bergen \And University of Bergen} ]

\begin{abstract}

Markov networks are  probabilistic graphical models that employ undirected graphs to depict conditional independence relationships among variables. Our focus lies in constraint-based structure learning, which entails learning the undirected graph from  data through the execution of conditional independence tests.
We establish theoretical limits concerning two critical aspects of constraint-based learning of Markov networks: the number of tests and the sizes of the conditioning sets. These bounds uncover an exciting interplay between the structural properties of the graph and the amount of tests required to learn a Markov network. 
The starting point of our work is that the graph parameter    \emph{maximum pairwise connectivity}, $\mvf$, that is, the maximum number of vertex-disjoint paths connecting a pair of vertices in the graph, 
is responsible for the sizes of independence tests required to learn the graph. On one hand, we show that at least one test with the 
size of the conditioning set at least $\mvf$ is always necessary. 
On the other hand, we prove that any graph can be learned by performing tests of size at most $\mvf$. 
This completely resolves the question of the minimum size of conditioning sets required to learn the graph. 
When it comes to the number of tests, our upper bound on the sizes of conditioning sets implies that every $n$-vertex graph can be learned by at most  $n^{\mvf}$ tests with conditioning sets of sizes at most $\mvf$.
We show that for any upper bound $q$ on the sizes of the conditioning sets, there exist graphs with $\Oh(n q)$ vertices that require at least $n^{\Omega(\mvf)}$ tests to learn.
This lower bound holds even when the treewidth and the maximum degree of the graph are at most $\mvf+2$.
On the positive side, we prove that every graph of bounded treewidth can be learned by a polynomial number of tests with conditioning sets of sizes at most~$2\mvf$. 
\end{abstract}

\section{INTRODUCTION}



Probabilistic graphical models (PGM) represent multivariate probability distributions using a graph structure to encode conditional independencies in the distribution. In addition to the graph structure, PGMs have parameters that specify the distribution. In this work, we study Markov networks whose structure is an undirected graph. 


Markov networks are usually learned using the so-called score-based approach, where one aims to find the structure and parameters that maximize a score (e.g., likelihood). Alternatively, one can use the constraint-based approach to learn the structure. In the constraint-based approach, one conducts conditional independence tests and constructs a graph expressing the same conditional independencies and dependencies implied by the test results. Note that if one uses a constraint-based approach to learn the structure, one has to learn parameters separately afterward.

The relation between the graph and conditional independencies is straightforward in a Markov network. If a distribution $P$ factorizes according to an undirected graph $G$ and vertices $u$ and $v$ are separated by a set $S$ in $G$, then $u$ and $v$ are conditionally independent given $S$ in $P$. Constraint-based learning aims in the opposite direction: One observes conditional independencies (or dependencies) in the distribution $P$, and the goal is to construct the graph $G$. 


This work aims to establish fundamental complexity results for constraint-based structure learning in Markov networks. Our results are based on two assumptions: (i) The distribution $P$ is faithful to an undirected graph $G$, that is, $u$ and $v$ are conditionally independent of $S$ in $P$ if and only if $u$ and $v$ are separated by $S$ in $G$ and (ii) we have access to a {\em conditional independence oracle} which always answers correctly to any pairwise conditional independence query in $P$. The first condition guarantees that there exists a unique graph $G$, and the second condition makes it possible to identify it.

While these assumptions are strong and the latter is never true in practice because statistical tests sometimes give erroneous results, they help us establish learning limits. In other words, if something cannot be learned under these idealized conditions, it cannot be learned in realistic settings. Under these assumptions, the constraint-based structure learning in Markov network reduces to the following elegant combinatorial model. In this model, for a vertex set  $V(G)$ of an unknown graph $G$, we want to learn all adjacencies between vertices of  $G$. 
 For that purpose, we use the independence oracle, which for any pair of vertices  $u, v\in V(G) $ and vertex set $S\subseteq V(G)$ correctly answers whether $S$ separates $u$ from $v$.

Of course, any graph on $n$ vertices could be learned by using $\Oh (n^2)$ queries of size  $n - 2$ just by asking for every pair of vertices if the remaining vertices of the graph separate them. However, assuming that the cost of asking the oracle grows very fast (like exponentially in the size of $S$),  we are interested in learning all adjacencies of $G$  by asking queries of the smallest possible size. This is motivated by statistical tests being less reliable and more computationally expensive when the conditioning set $S$ is large.

A well-known observation is that the structure of a Markov network could be learned from conditioning sets. In a Markov network, a variable is conditionally independent of its non-neighbors given its neighbors (Markov blanket). This brings to the observation that a graph with the maximum vertex degree $\Delta$ could be learned with   $\Oh(n^\Delta)$ independence tests with conditioning set $S$ of size at most $\Delta$; see, for example, \citet{koller}. 


The maximum degree $\Delta$ serves as an example of a structural property within the structure $G$ of the data-generating distribution $P$. This naturally leads us to the following question: \emph{How do the structural properties of the data-generating distribution $P$ impact the number of conditional independence tests and the sizes of conditioning sets required to reconstruct the graph structure $G$?} Furthermore, \emph{what other properties, apart from the maximum degree $\Delta$, are crucial for the learning process?}


Our contributions are as follows. First, we address the question about the size of the conditioning sets. We identify the parameter \emph{maximum pairwise connectivity $\mvf$}, the maximum number of vertex disjoint paths connecting a pair of vertices of the graph, as the key factor. To define it more precisely,  
for two vertices $u,v \in V(G)$, we denote by $\mvf(G, u, v)$ the maximum number of vertex-disjoint paths, each having at least one internal vertex, between $u$ and $v$.
Then $\mvf(G) = \max_{u,v \in V(G)} \mvf(G, u, v)$  is the maximum value of $\mvf(G, u, v)$ over all pairs $u,v$, see \Cref{fig:kappa}. It's important to note that the parameter $\mvf$ never exceeds the maximum vertex degree $\Delta$. Additionally, there are graphs, like trees, where 
$\mvf$ is equal to one and $\Delta$ can be as large as the total number of vertices minus one. 

We demonstrate that in order to identify the graph $G$, it is necessary to perform at least one conditional independence test with a conditioning set of size $\mvf$, as stated in \Cref{the:mainlb}. Furthermore, we establish that a conditioning set size of $\mvf$ is sufficient. In other words, any graph can be learned through tests involving conditioning sets of size at most $\mvf$, as proved in \Cref{the:basicalgo}.

The proof of \Cref{the:basicalgo} is constructive, and the upper bound is achieved using a straightforward algorithm. Essentially,  this algorithm  conducts all possible independence tests up to size $ \mvf$. Together, these bounds demonstrate that the size of the largest independence test needed is determined  by the structure of the graph $G$. Moreover, no algorithm can outperform the straightforward algorithm in this regard.

%
%

\begin{figure}
   \centering

\begin{tikzpicture}
  \node[draw,circle] (center) at (0,0) {$v_0$};
  
  \foreach \i in {1,2,3,4,5}{
    \node[draw,circle] (v\i) at (72*\i:2) {$v_\i$};
    \draw (center) -- (v\i);
    \ifnum\i>1
      \pgfmathtruncatemacro{\previ}{\i-1}
      \draw (v\previ) -- (v\i);
  
    \fi
       }
  
  \draw   (center) -- (v5);
   \draw  (v1) -- (v5);
  \foreach \i in {2,3,4,5}{
    \node[draw,circle] (v\i) at (72*\i:2) {$v_\i$};
 \draw  (center) -- (v3);

       }
\end{tikzpicture}

      \caption{$\mvf(G, v_0, v_1)=2$:  Vertices $v_0,v_1$ are connected by two non-trivial vertex disjoint paths, $v_0, v_2,v_1$  and $v_0, v_5,v_1$. There is a pair with larger connectivity, namely $\mvf(G, v_1, v_4)=3$.  Vertices $v_1,v_4$ are connected by three non-trivial vertex disjoint paths, namely, $(v_1, v_0,v_4)$, $(v_1, v_2,v_3, v_4)$,   and $(v_1, v_5, v_4)$. The maximum pairwise connectivity $\mvf$ of this graph is $3$.  The maximum vertex degree $\Delta$ is 5. }
    \label{fig:kappa}
\end{figure}
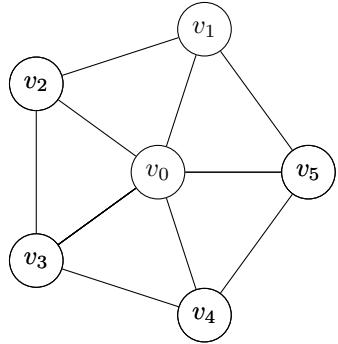

Next, we delve into the question of how many conditional independence tests are required. It becomes evident that the maximum pairwise connectivity plays a significant role in this context as well. We establish that in certain scenarios, it becomes necessary to conduct as many as $|V(G)|^{\Omega (\kappa)}$ tests to identify the graph $G$, as outlined in \Cref{the:xplowerbound}.
Once again, the naive algorithm employed in the proof of \Cref{the:basicalgo} requires, at most, $|V(G)|^{\Oh(\mvf)}$ tests. Consequently, in the worst-case scenario, no algorithm can significantly reduce the number of required tests.
On the positive side, we show that if the treewidth $\tw$ of $G$ is much smaller than $\mvf$, then $|V(G)|^{\tw}$ conditional independence tests with conditioning sets of size at most $2 \mvf$ suffice (\Cref{the:twubmain}).






{\bf Related work.} A Markov network, which is a tree (the network of treewidth 1) can be learned in polynomial time \citep{chow68} using score-based methods. However, the problem becomes NP-hard for any other treewidth bound \citep{karger01}. It has been shown that Markov networks are PAC-learnable in polynomial time by using constraint-based algorithms \citep{abbeel,chechetka07,narasimhan04}. To the best of our knowledge,
our work provides the first results on constraint-based structure learning of Markov networks beyond the classic $\Oh(n^\Delta)$ bound (see, e.g., \citep{koller}).
  
 Constraint-based structure learning is actively studied in other PGMs, such as Bayesian networks. Typically, one uses the PC algorithm \citep{spirtes} or one of its variants (e.g., \citep{DBLP:conf/pgm/AbellanGM06, giudice}), which learns an undirected skeleton first and then directs the edges. Constraint-based learning of Bayesian networks has also been studied with structural properties such as treewidth \citep{pmlr-v138-talvitie20a}. The most relevant work in this context is the variation of the  PC algorithm proposed by   
  \citet{DBLP:conf/pgm/AbellanGM06}. This heuristic for Bayesian networks exploits small cuts and could speed up learning in practice. In spirit, it is close to the parameter $\kappa$ we define here. The crucial difference here is that for Markov networks, we can \emph{guarantee} theoretically that small connectivity helps (\Cref{the:basicalgo}) to learn the network. 
  In sharp contrast to this result, it is not difficult to come out with the ``worst-case'' examples of Bayesian networks when small $\kappa$ or any other type of connectivity, does not provide any advantages in learning the network. This is due to the $d$-separation criterion and presence of $v$-structures.




\section{NOTATION}


For integers $a$ and $b$, we denote by $[a,b]$ the set of integers $\{a,a+1, \ldots,b-1,b\}$, and by $[a]$ the set $[1,a]$.
For a graph $G$, we denote by $V(G)$ the set of its vertices and $E(G)$ the set of its edges.
For a set $S \subseteq V(G)$, we denote by $G[S]$ the subgraph of $G$ induced by $S$, and by $G \setminus S$ the subgraph of $G$ induced by $V(G) \setminus S$.
We denote by $N(S) \subseteq V(G) \setminus S$ the set of neighbors of vertices in $S$ that are outside of $S$.
We use the convention that a connected component $C$ of a graph is a set of vertices $C \subseteq V(G)$.
A set of vertices $S \subseteq V(G)$ is an \emph{$u$-$v$-separator} if $u$ and $v$ are in different connected components of $G \setminus S$.
We denote by $G - uv$ the graph $G$ with the edge between vertices $u$ and $v$ removed, and by $G + uv$ the graph $G$ with the edge between vertices $u$ and $v$ added.

A \emph{tree decomposition} of a graph $G$ is a pair $(T,\bag)$, where $T$ is a tree and $\bag : V(T) \rightarrow 2^{V(G)}$ is a function assigning each node of $T$ a subset of vertices called a bag, so that $(T,\bag)$ satisfies the following three properties: (1) $V(G) = \bigcup_{t \in V(T)} \bag(t)$, (2) for every edge $uv \in E(G)$, there exists $t \in V(T)$ so that $\{u,v\} \subseteq \bag(t)$, and (3) for every vertex $v \in V(G)$, the subtree $T[\{t \in V(T) \colon v \in \bag(t)\}]$ induced by the bags containing $v$ is connected~\citep{DBLP:journals/jal/RobertsonS86}.
The \emph{width} of a tree decomposition is $\max_{t \in V(T)} |\bag(t)|-1$, and the treewidth of a graph is a minimum width of a tree decomposition of it.
We use $\tw(G)$ to denote the \emph{treewidth} of $G$.
A \emph{path decomposition} is a tree decomposition where the tree $T$ is a path.
The width of a path decomposition and the pathwidth $\pw(G)$ of a graph $G$ are defined analogously. See  \Cref{fig:treedecomp} for an example of a path decomposition. 
\begin{figure}[htb!]
    \centering
\includegraphics[width=0.4\textwidth]{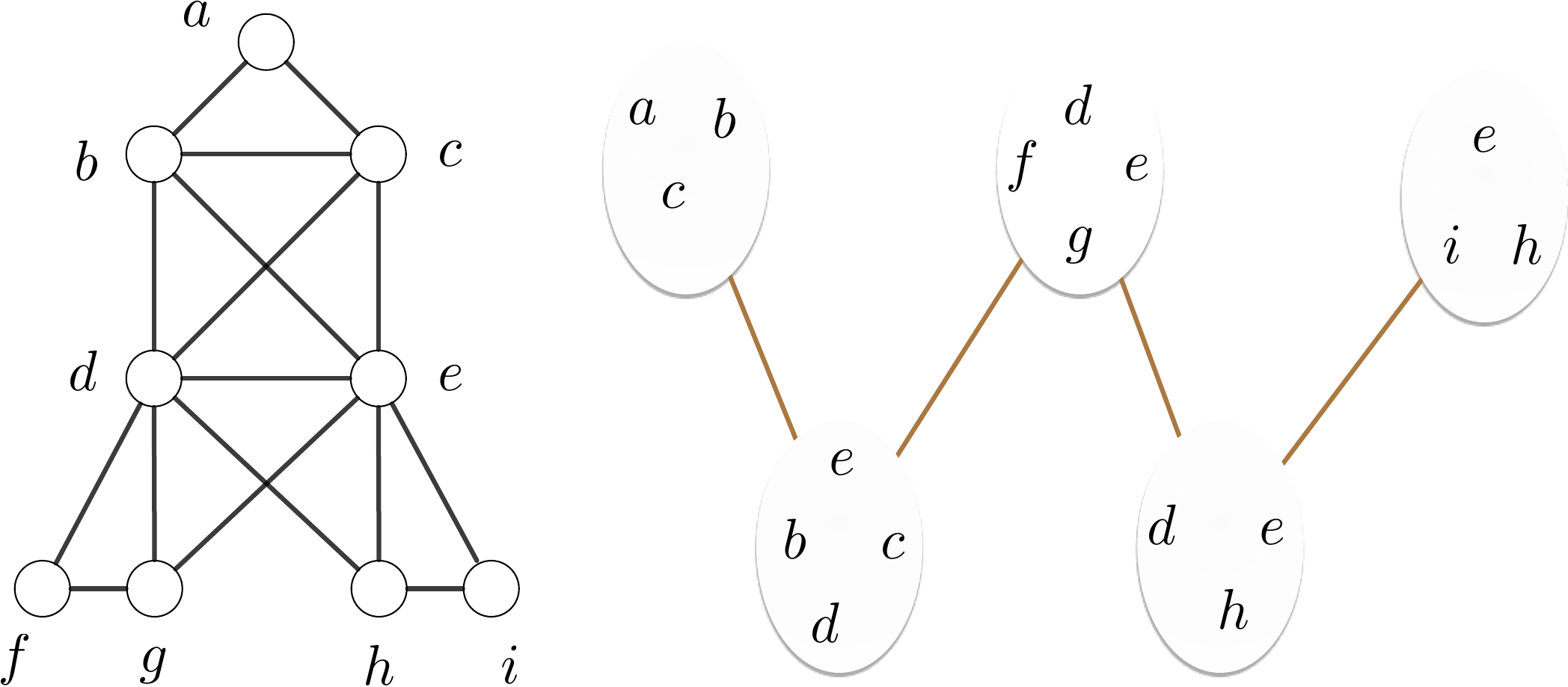}  
    \caption{Example of a path decomposition of width 3 of a graph.}
    \label{fig:treedecomp}
\end{figure}

For two vertices $u,v \in V(G)$, we denote by $\mvf(G, u, v)$ the maximum number of vertex-disjoint paths, each having at least one internal vertex, between $u$ and $v$.
By Menger's theorem, $\mvf(G, u, v)$ equals the size of the smallest $u$-$v$-separator in the graph $G - uv$.
We use $\mvf(G) = \max_{u,v \in V(G)} \mvf(G, u, v)$ to denote the maximum value of $\mvf(G, u, v)$ over all pairs $u,v$.
We call the parameter $\mvf(G)$ the \emph{\parvf} of $G$.
We denote the maximum degree of a graph $G$ by $\Delta(G)$.
Observe that $\mvf(G) \le \Delta(G)$.

We denote an independence test on an underlying graph $G$ by a triple $(S, u, v)$, where $S \subseteq V(G)$ and $u, v \in V(G) \setminus S$, to which the conditional independence oracle answers \connreply if $u$ and $v$ are in the same connected component of $G \setminus S$ and \disconnreply if $u$ and $v$ are in different connected components of $G \setminus S$.
The size of an independence test $(S, u, v)$ is $|S|$.

\section{LOWER AND UPPER BOUNDS BY \Parvf}
This section identifies the maximum pairwise connectivity $\mvf(G)$ as the fundamental parameter.
First, we show that to learn an underlying graph $G$, one must conduct at least one independence test of size at least $\mvf(G)$.

\begin{theorem}
\label{the:mainlb}
For any graph $G$, at least one independence test of size at least $\mvf(G)$ is required to decide if the underlying graph is equal to $G$.
\end{theorem}
\begin{proof}
Let $u,v \in V(G)$ be vertices so that $\mvf(G, u, v) = \mvf(G)$.
Now, if the graph $G$ contains the edge $uv$, define $G'$ to be the graph $G - uv$, and if the graph $G$ does not contain the edge $uv$, define $G'$ to be the graph $G + uv$.
We observe that for any set $S \subseteq V(G)$ that is disjoint from $\{u,v\}$ and has size $|S| < \mvf(G)$, the vertices $u$ and $v$ are in the same connected component in the graph $G \setminus S$, and also in the graph $G' \setminus S$.
This implies that for any set $S \subseteq V(G)$ of size $|S| < \mvf(G)$, the connected components of $G \setminus S$ and $G' \setminus S$ are the same.
It follows that for any independence test of size less than $\mvf(G)$, the answer will be the same whether the underlying graph is $G$ or $G'$, and therefore, as $G$ and $G'$ are not the same graph, any algorithm using only independence tests of size less than $\mvf(G)$ cannot correctly decide if the underlying graph is equal to $G$.
\end{proof}

While the formulation of \Cref{the:mainlb} excludes a decision algorithm that decides whether the underlying graph is equal to $G$, this also implies the same lower bound for learning the underlying graph, because the decision problem can be solved by learning the graph.
Note also that the lower bound holds not only for a worst-case graph $G$, but for every individual graph $G$.
It also follows from the proof that by using independence tests of size less than $\mvf(G)$, even basic properties such as the number of edges of the underlying graph $G$ cannot be determined.

Then we show that in order to learn the underlying graph $G$, it is sufficient to conduct independence tests of size at most $\mvf(G)$.
Our algorithm does not even need to know the value $\mvf(G)$ in advance, but it can determine it without conducting independence tests of larger size than $\mvf(G)$.

\begin{theorem}
\label{the:basicalgo}
There is an algorithm that learns the underlying graph $G$ by using at most $|V(G)|^{\mvf(G)+2}$ independence tests of size at most $\mvf(G)$.
\end{theorem}
\begin{proof}
Let $k$ be a non-negative integer.
We will show that by conducting all $|V(G)|^{k+2}$ possible independence tests of size $\le k$, we can either learn the underlying graph $G$ or conclude that $\mvf(G) > k$.
Then, the algorithm works by starting with $k=0$ and increasing $k$ iteratively, at each iteration conducting all of the independence tests of size $k$, until it concludes what the underlying graph is.
Note that it never conducts independence tests of size more than $\mvf(G)$.

It remains to show that given answers to all possible independence tests of size $\le k$, we can either learn the underlying graph $G$  or conclude that $\mvf(G) > k$.
We create a graph $G^k$ with vertex set $V(G^k) = V(G)$ and edge set so that for each pair of vertices $u,v \in V(G^k)$, there is an edge between $u$ and $v$ if there exists no independence test $(S, u, v)$ of size $|S| \le k$ with the answer \disconnreply.
If there is an edge between $u$ and $v$ in $G$, then there is an edge between $u$ and $v$ in $G^k$, so $G^k$ is a supergraph of $G$.
Next, we show that if $k \ge \mvf(G)$, the converse also holds.

\begin{claim}
\label{the:basicalgo:claim1}
If $k \ge \mvf(G)$, then $G = G^k$.
\end{claim}
\begin{claimproof}
By the observation that $G^k$ is a supergraph of $G$, it remains to prove that if there is no edge between $u$ and $v$ in $G$ then there is no edge between $u$ and $v$ in $G^k$.
If there is no edge between $u$ and $v$ in $G$, by Menger's theorem, there exists an $u$-$v$-separator $S$ of size $|S| = \mvf(G, u, v) \le k$.
Therefore, the independence test $(S, u, v)$ has size at most $k$ and has the answer \disconnreply, so there is no edge $uv$ in the graph $G^k$.
\end{claimproof}

Now, we make the decision as follows.
If $\mvf(G^k) \le k$, we conclude that the underlying graph is $G^k$.
Otherwise, we conclude that $\mvf(G) > k$.
The correctness of the first conclusion follows from the fact that because $G^k$ is a supergraph of $G$, we have that $\mvf(G) \le \mvf(G^k) \le k$, and therefore by \Cref{the:basicalgo:claim1} that $G^k = G$.
The correctness of the second conclusion follows from the fact that if $\mvf(G) \le k$ would hold, then $G = G^k$ would hold by \Cref{the:basicalgo:claim1} and therefore also $\mvf(G^k) = \mvf(G) \le k$ would hold.
\end{proof}

\section{A LOWER BOUND FOR THE NUMBER OF INDEPENDENCE TESTS}
By \Cref{the:mainlb}, to learn a graph $G$, one has to perform at least one independence test of size at least $\mvf(G)$.
However, it is not clear if $|V(G)|^{\mvf(G)+2}$ independence tests made by the algorithm of \Cref{the:basicalgo} are necessary, or if $G$ could be learned by a significantly smaller number of independence tests of size roughly $\mvf(G)$.
We show that $|V(G)|^{\Omega(\mvf(G))}$ independence tests are required in some cases, even if we allow for much larger tests than of size $\mvf(G)$.
Our lower bound holds even under several additional structural restrictions.

A proper interval graph is a graph that can be represented as an intersection graph of intervals on a line, where no interval properly contains another.
Note that every proper interval graph is also a chordal graph.

\begin{theorem}
\label{the:xplowerbound}
Let $n$, $q$, and $k$ be integers with $n \ge q \ge k \ge 3$.
There exists a proper interval graph $G$ with $|V(G)| = \Oh(n q)$ vertices, $\mvf(G) = 2k-2$, pathwidth $\pw(G) = k$, and maximum degree $\Delta(G) = 2k$, so that any algorithm using independence tests of size at most $q$ requires at least $n^{k-3}$ independence tests to decide if the underlying graph is isomorphic to $G$.
\end{theorem}
The lower bound in \Cref{the:xplowerbound} holds even when we are given a promise that the underlying graph $G'$ is a proper interval graph with $\mvf(G') = 2k-2$, $\pw(G') = k$, and $\Delta(G') = 2k$.
\begin{proof}
Let $G$ be the graph with $3 q n$ vertices $v_1, v_2, \ldots, v_{3qn}$, so that a vertex $v_i$ is adjacent to a vertex $v_j$ whenever $|i-j| \le k$.
It is easy to observe that $G$ is a proper interval graph.
Also, we can observe the facts that $\Delta(G) = 2k$, $\pw(G) = k$, and $\mvf(G) = 2k-2$:
Any vertex $v_i$ with $i \in [k+1, 3qn-k]$ has degree exactly $2k$ and other vertices have degrees less than $2k$.
The sequence of bags $\{v_1, \ldots, v_{k+1}\}, \{v_2, \ldots, v_{k+2}\}, \ldots, \{v_{3qn-k}, \ldots, v_{3qn}\}$ readily gives a path decomposition of width $k$ where each bag is a clique, certifying that the pathwidth is exactly $k$.
To prove $\mvf(G) \le 2k-2$, consider $v_i,v_j \in V(G)$ with $i<j$.
If $j-i > k$, then $\{v_{i+1}, \ldots, v_{i+k}\}$ is a $v_i$-$v_j$-separator of size $k \le 2k-2$.
If $j-i \le k$, then $\{v_{i+1}, \ldots, v_{i+k}\} \cup \{v_{j-1}, \ldots, v_{j-k}\} \setminus \{v_i,v_j\}$ is a $v_i$-$v_j$-separator in $G - v_i v_j$ of size at most $2k-2$.
To prove $\mvf(G) \ge 2k-2$, for any $i \in [k,3qn-k]$ we can construct $2k-2$ vertex-disjoint paths, each with exactly one internal vertex, between $v_i$ and $v_{i+1}$.

We let $G^{-}$ be the graph $G - v_1 v_{k+1}$, i.e., the graph $G$ but with the edge between the vertices $v_1$ and $v_{k+1}$ removed.
Again, by similar arguments we can observe that $G^{-}$ is a proper interval graph with $\mvf(G^{-}) = 2k-2$, $\Delta(G^{-}) = 2k$, and $\pw(G^{-}) = k$.

Consider any algorithm that uses independence tests of size at most $q$, and always makes less than $n^{k-3}$ independence tests to decide whether the underlying graph is isomorphic to $G$.
We consider an adversary that simply always answers \connreply.
It remains to prove that there exists a graph isomorphic to $G$ that is consistent with all of the answers and also a graph isomorphic to $G^{-}$ that is consistent with all of the answers.
As $G$ and $G^{-}$ are not isomorphic ($G^{-}$ has fewer edges than $G$), this would imply that the algorithm cannot decide whether the underlying graph is isomorphic to $G$ or not, and therefore cannot be correct.

To prove this, we will first prove an auxiliary claim.
Let $S \subseteq V(G)$.
For integers $a$ and $b$ with $1 \le a \le b \le 3qn$, we say that $S$ covers the interval $[a,b]$ if it holds that $\{v_a, v_{a+1}, \ldots, v_b\} \subseteq S$.
\begin{claim}
\label{the:xplowerbound:claimconn}
Let $S \subseteq V(G)$.
If $S$ does not cover any interval with $k-1$ vertices, then the graph $G^{-} \setminus S$ is connected.
\end{claim}
\begin{claimproof}
Let $S \subseteq V(G)$ be a set of vertices that do not cover any interval with $k-1$ vertices.
Let $i \in [3qn]$ be the smallest index so that $v_i \in V(G) \setminus S$.
We prove by contradiction that every vertex in $G^{-} \setminus S$ is in the same connected component as $v_i$, implying that $G^{-} \setminus S$ is connected.
For the sake of contradiction, let $j$ be the smallest index so that $v_j \in V(G) \setminus S$ and $v_j$ is in a different connected component of $G^{-} \setminus S$ than $v_i$.
Note that $j > i$.

First, consider the case when $j-i < k$.
In this case, there is an edge between $v_i$ and $v_j$, implying that they are in the same connected component of $G^{-} \setminus S$, which is a contradiction.

Then, consider the case when $j-i \ge k$.
In this case, the interval $[j-k+1, j-1]$ contains $k-1$ vertices, so $S$ does not cover it, so there is a vertex $v_l \in V(G) \setminus S$ with $l \in [j-k+1, j-1]$.
By the fact that we chose $j$ to be the smallest index such that $v_j \in V(G) \setminus S$ and $v_j$ is in a different connected component of $G^{-} \setminus S$ than $v_i$, it holds that $v_l$ is in the same connected component of $G^{-} \setminus S$ as $v_i$.
However, there is an edge between $v_l$ and $v_j$, so $v_j$ is also in the same component, which is a contradiction.
\end{claimproof}

Let $\pi : V(G) \rightarrow V(G)$ be a permutation of $V(G)$.
We denote by $\pi(G)$ the graph isomorphic to $G$ resulting from applying $\pi$ to the vertices and edges of $G$.
Now, our goal is to show that there exists a permutation $\pi$, so that both $\pi(G)$ and $\pi(G^{-})$ are consistent with always answering \connreply.
\begin{claim}
Let $(S_1, u_1, v_1), \ldots, (S_t, u_t, v_t)$ be a sequence of $t < n^{k-3}$ independence tests.
There exists a permutation $\pi$ of $V(G)$ so that both of the graphs $\pi(G)$ and $\pi(G^{-})$ are consistent with answering \connreply to all of the independence tests.
\end{claim}
\begin{claimproof}
For an independence test $(S_i, u_i, v_i)$ and a permutation $\pi$ of $V(G)$, denote by $\pi(S_i, u_i, v_i) = (\pi(S_i), \pi(u_i), \pi(v_i))$ the application of a permutation $\pi$ to all of the vertices specified in the independence test.
We observe that to prove the claim it suffices to prove that there exists a permutation $\pi$ so that both of the graphs $G$ and $G^{-}$ are consistent with answering \connreply to all independence tests $\pi(S_1, u_1, v_1), \ldots, \pi(S_t, u_t, v_t)$.
Indeed, the answer to the independence test $\pi(S_i, u_i, v_i)$ on an underlying graph $H$ is the same as the answer to the independence test $(S_i, u_i, v_i)$ on an underlying graph $\pi^{-1}(H)$, where $\pi^{-1}$ denotes the inverse permutation of $\pi$.

To prove the existence of such a permutation $\pi$, we show that when $\pi$ is selected uniformly randomly among all permutations of $V(G)$, there is a non-zero probability that $G$ and $G^{-}$ are consistent with answering \connreply to all of the independence tests $\pi(S_1, u_1, v_1), \ldots, \pi(S_t, u_t, v_t)$.

Let $\pi$ be a permutation selected uniformly randomly among all permutations of $V(G)$ and let $(S_i, u_i, v_i)$ be a fixed independence test.
Note that by \Cref{the:xplowerbound:claimconn}, $G^{-}$ is consistent with answering \connreply to $\pi(S_i, u_i, v_i)$ if $\pi(S_i)$ does not cover any interval with $k-1$ vertices.

Now, let $[a,b]$ be an interval with $b-a+1 = k-1$ vertices and $(S_i, u_i, v_i)$ an independence test of size $k-1 \le |S_i| \le q$.
We observe that

\begin{align*}
\Pr[\pi(S_i) \text{ covers } [a,b]] =& \frac{|S_i|}{3qn} \cdot \frac{|S_i|-1}{3qn-1} \cdot \ldots \cdot \frac{|S_i|-k+2}{3qn-k+2}\\
\le& \left(\frac{q}{3qn-k+2}\right)^{k-1} \le \left(\frac{q}{2qn}\right)^{k-1} \\ \le& \left(\frac{1}{2n}\right)^{k-1}.
\end{align*}
Now, because there are $3qn-k+2$ intervals with $k-1$ vertices, the expected number of intervals with $k-1$ vertices that $\pi(S_i)$ covers is at most
\begin{align*}
(3qn-k+2) \cdot \left(\frac{1}{2n}\right)^{k-1} \le qn \cdot \left(\frac{1}{n}\right)^{k-1} \le \left(\frac{1}{n}\right)^{k-3}.
\end{align*}
The expected sum of the numbers of covered intervals with $k-1$ vertices over all of the $t < n^{k-3}$ independence tests is at most 
\begin{align*}
t \cdot \left(\frac{1}{n}\right)^{k-3} < n^{k-3} \cdot \left(\frac{1}{n}\right)^{k-3} < 1.
\end{align*}
Because the expected number is less than one, there is a non-zero probability that all independence tests cover no intervals with $k-1$ vertices.
In particular, there must exist a permutation $\alpha$ so that none of the independence tests $\pi(S_1, u_1, v_1), \ldots, \pi(S_t, u_t, v_t)$ covers an interval with $k-1$ vertices, and therefore all of the graphs $G^{-} \setminus S_i$ are connected, and therefore $G^{-}$ is consistent with answering \connreply to every independence test.
Because $G^{-}$ is a subgraph of $G$, all of the graphs $G \setminus S_i$ are also connected, and therefore $G$ is consistent with answering \connreply to every independence test.
\end{claimproof}
As the graph $\pi(G)$ is isomorphic to $G$ and $\pi(G^{-})$ is isomorphic to $G^{-}$, this finishes the proof.
\end{proof}

Note that while the statement of \Cref{the:xplowerbound} excludes a decision algorithm for deciding if the underlying graph is isomorphic to $G$, the same lower bound also holds for learning the underlying graph because the decision problem can be solved by learning the graph.

\section{TREEWIDTH}

In this section, we give an improved upper bound for the setting where the treewidth $\tw(G)$ is smaller than $\mvf(G)$.
First, we show that an optimum-width tree decomposition of $G$ can be learned by using independence tests of size at most $\tw(G)$.
For this, will need the following lemma from~\cite{DBLP:journals/jco/Bodlaender03}.

\begin{lemma}[\cite{DBLP:journals/jco/Bodlaender03}]
\label{lem:twlem}
Let $G$ be a graph, $u,v \in V(G)$, and $k$ an integer.
If $\mvf(G, u, v) > k$, then every tree decomposition of $G$ of width at most $k$ contains a bag that contains both $u$ and $v$.
\end{lemma}

We show that \Cref{lem:twlem} can be harnessed to learn an optimum-width tree decomposition of the underlying graph $G$ by independence tests of size at most $\tw(G)$.
This is similar to the proof of \Cref{the:basicalgo}.

\begin{theorem}
\label{the:twcomput}
There is an algorithm that learns a tree decomposition of the underlying graph $G$ of width at most $\tw(G)$ by using at most $|V(G)|^{\tw(G)+2}$ independence tests of size at most $\tw(G)$.
\end{theorem}
\begin{proof}
Let $k$ be a non-negative integer.
We show that by conducting all $|V(G)|^{k+2}$ possible independence tests of size $\le k$, we can either conclude that $\tw(G) \le k$ and output a tree decomposition of width at most $k$ of the underlying graph $G$, or to conclude that $\tw(G) > k$.
Then, the algorithm works by starting with $k=0$ and increasing $k$ iteratively, at each iteration conducting all of the independence tests of size $k$, until it concludes with a tree decomposition of minimum width.
Note that it never conducts independence tests of size more than $\tw(G)$.

It remains to show that given answers to all possible independence tests of size $\le k$, we can either conclude with a tree decomposition of the underlying graph of width at most $k$, or that $\tw(G) > k$.
We create a graph $G^k$, so that for each pair of vertices $u, v \in V(G)$ there is an edge between $u$ and $v$ if there exists no independence test $(S,u,v)$ of size $|S| \le k$ with the answer \disconnreply.
Clearly, $G^k$ is a supergraph of $G$, and therefore $\tw(G^k) \ge \tw(G)$ and any tree decomposition of $G^k$ is also a tree decomposition of $G$.
We use the $|V(G)|^{k+2}$ time algorithm of~\citet{ArnborgCP87} to either compute a tree decomposition of $G^k$ of width at most $k$ or to decide that $\tw(G^k) > k$.
When we get a tree decomposition of width at most $k$, we return this tree decomposition.
In the case when $\tw(G^k) > k$, we conclude that $\tw(G) > k$.
Let us prove that this conclusion is correct.

Suppose $G$ has a tree decomposition $(T,\bag)$ of width at most $k$.
For any edge $uv \in E(G^k) \setminus E(G)$, it holds that $\mvf(G, u, v) > k$, and therefore by \Cref{lem:twlem}, $(T,\bag)$ contains a bag that contains both $u$ and $v$.
Therefore, $(T,\bag)$ is also a tree decomposition of $G^k$, so $\tw(G^k) \le k$, but this is a contradiction.
\end{proof}

Then we show that \Cref{the:twcomput} can be leveraged to learn the underlying graph $G$ by conducting $|V(G)|^{\Oh(\tw(G))}$ independence tests of size $\Oh(\tw(G))$ plus only a polynomial number of independence tests of size at most $2 \cdot \mvf(G)$.

\begin{theorem}
\label{the:twubmain}
There is an algorithm that learns the underlying graph $G$ by using at most $|V(G)|^{\Oh(\tw(G))}$ independence tests of size at most $\Oh(\tw(G))$ and at most $|V(G)| \cdot \tw(G)$ independence tests of size at most $2 \cdot \mvf(G)$.
\end{theorem}
\begin{proof}
We assume that $\tw(G) \ge 1$, as otherwise, $G$ has no edges, and we can solve the problem by $|V(G)|^2$ independence tests of size $0$.
We start by using the algorithm of \Cref{the:twcomput} to compute a tree decomposition $(T,\bag)$ of $G$ of width $\tw(G)$.
Let $u,v \in V(G)$ be a pair of vertices that occurs in some bag of $(T,\bag)$ together.
Because the graph that contains edges between all such pairs has treewidth at most $\tw(G)$, and any graph with $n$ vertices and treewidth $k$ has at most $n \cdot k$ edges, there are at most $|V(G)| \cdot \tw(G)$ such pairs.
Therefore, we focus on giving an algorithm that learns whether $uv \in E(G)$ by using $|V(G)|^{\Oh(\tw(G))}$ independence tests of size $\Oh(\tw(G))$, and only one independence test of size at most $2 \cdot \mvf(G)$.

Now, our goal is to find a set $S \supseteq \{u,v\}$, of size $|S| \le \mvf(G)$, so that for each connected component $C$ of $G \setminus S$, it holds that $\mvf(G[C \cup \{u,v\}], u, v) \le \Oh(\tw(G))$ and $|N(C)| \le \Oh(\tw(G))$.
We need the following auxiliary claim for finding such a set $S$ and then for using it.

\begin{claim}
\label{the:twubmain:claimmincut}
There is an algorithm that given two disjoint vertex subsets $C,S \subseteq V(G)$ with $N(C) \subseteq S$ and $u,v \in S$, and an integer $f$, determines whether $\mvf(G[C \cup \{u, v\}], u, v) \le f$ by using $\Oh(|V(G)|^f)$ independence tests of size at most $|S|+f-1$.
If $\mvf(G[C \cup \{u, v\}], u, v) \le f$ holds, the algorithm also outputs a set $X \subseteq C$ of size $|X| \le f$ so that $X$ separates $u$ from $v$ in the graph $G[C \cup \{u, v\}] - uv$.
\end{claim}
\begin{claimproof}
Given $X \subseteq C$ and $w \in C \setminus X$, we can compute if $u$ is reachable from $w$ in $G[C \cup \{u, v\}] \setminus X$ by conducting the independence test $(X \cup S \setminus \{u\}, w, u)$, and if $v$ is reachable from $w$ by conducting the independence test $(X \cup S \setminus \{u\}, w, v)$.
Therefore, given $X \subseteq C$, we can with $|C|-|X|$ independence tests of size $|X|+|S|-1$ determine whether $X$ separates $u$ from $v$ in the graph $G[C \cup \{u, v\}] - uv$.
By Menger's theorem, the smallest size of such $X$ is equal to $\mvf(G[C \cup \{u, v\}], u, v)$, so by testing all subsets $X$ of $C$ of size at most $f$ we can test if $\mvf(G[C \cup \{u, v\}], u, v) \le f$.
\end{claimproof}

Then, we describe how to find the desired set $S$.
We will treat the tree $T$ of the decomposition as rooted at an arbitrarily selected node, and use the standard rooted tree terminology.
Moreover, by adding the root as an additional node, we can assume that the bag of the root is empty.
For a set of nodes $R \subseteq V(T)$, we denote by $\bag(R)$ the union of their bags.
We will find $R \subseteq V(T)$, so that $|R| \le \mvf(G)/(\tw(G)+1)-1$, and the set $S = \bag(R) \cup \{u,v\}$ satisfies the required properties.
Moreover, the set $R$ will be \emph{LCA-closed}: for any pair of nodes in $R$, their lowest common ancestor (LCA) is also in $R$.
This implies that for any connected component $C$ of $G \setminus (\bag(R) \cup \{u,v\})$, its neighborhood $N(C)$ is a subset of a union of two bags and $\{u,v\}$, and in particular has size $|N(C)| \le 2 \tw(G)+4$.

To find the set $R$, we proceed as follows.
For a node $t \in V(T)$, we denote by $\cmp(t) \subseteq V(G) \setminus \bag(t)$ the vertices that occur in the bags of the subtree rooted at $t$, but not in $\bag(t)$. 
We process the nodes of $T$ in any order so that all descendants of a node are processed before it, for example, in the depth-first-order.
At each node $t$ we are processing, we will decide, based on the following three cases, whether to add $t$ to $R$.
\begin{enumerate}
\item\label{the:twubmain:rconst1} If $t$ is the lowest common ancestor of two nodes already in $R$, we add $t$ to $R$.
\item\label{the:twubmain:rconst2} If no descendant of $t$ is in $R$, we add $t$ to $R$ if $\mvf(G[\cmp(t) \cup \{u,v\}], u, v) \ge 2\tw(G)+2$.
\item\label{the:twubmain:rconst3} If $t$ has a descendant in $R$ but (\ref{the:twubmain:rconst1}) does not apply, there is a unique descendant $l \in R$ of $t$ that is the lowest common ancestor of the set of all descendants of $t$ in $R$. Now, let $C = \cmp(t) \setminus (\cmp(l) \cup \bag(l))$.
We add $t$ to $R$ if $\mvf(G[C \cup \{u,v\}], u, v) \ge 2\tw(G)+2$.
\end{enumerate}

First, to argue that this process of selecting $R$ can be implemented with $|V(G)|^{\Oh(\tw(G))}$ independence tests of size $\Oh(\tw(G))$, observe that in case (2) we have that $N(\cmp(t)) \subseteq \bag(t)$, and in case (3) we have that $N(C) \subseteq \bag(t) \cup \bag(l)$.
Therefore, we can use the algorithm of \Cref{the:twubmain:claimmincut} for checking if $\mvf(G[C \cup \{u,v\}], u, v) \ge 2\tw(G)+2$ in these cases, by using at most $|V(G)|^{2 \tw(G)+2}$ independence tests of size at most $4 \tw(G) + 4$.

Then, we wish to show that $\mvf(G[C \cup \{u,v\}]) \le 3 \tw(G) + 2$ for each connected component $C$ of $G \setminus \bag(R)$.
Suppose that there exists a component $C$ with $\mvf(G[C \cup \{u,v\}]) > 3 \tw(G) + 2$, and let $t$ be the lowest node for which it holds that $\mvf(G[(C \cap \cmp(t)) \cup \{u,v\}] > 3 \tw(G) + 2$.
Such $t$ exists because $\cmp(r) = V(G)$ for the root node $r$.
By the fact that $C$ intersects a connected subtree of bags of $(T,\bag)$ and by our construction, it must hold that $t \in R$.
Now, a unique child $c$ of $t$ exists such that $C \subseteq \cmp(c) \cup \bag(c)$.
Because $t$ was the lowest node such  that $\mvf(G[(C \cap \cmp(t)) \cup \{u,v\}] > 3 \tw(G) + 2$, the set $C$ must intersect $\bag(c)$. Hence $c \notin R$.
However, since $|\bag(c)| \le \tw(G)+1$, we have that $\mvf(G[(C \cap \cmp(c)) \cup \{u,v\}] \ge 2 \tw(G) + 2$. This contradicts our construction, in particular, the node $c$ should have been selected to $R$ in this case.

Finally, we argue that $|R| \le \mvf(G)/(\tw(G)+1)-1$.
First, observe that if a node $t$ is added to $R$ in the cases (\ref{the:twubmain:rconst2}) or (\ref{the:twubmain:rconst3}), then $\bag(t)$ separates the considered component $C$ (in the case (\ref{the:twubmain:rconst2}) let $C = \cmp(t)$) from the nodes not processed yet, which implies that the sets $C$ considered in such cases are disjoint.
Note also that each such set $C$ is disjoint from the bags of the nodes later added to $R$. In particular, $C \subseteq V(G) \setminus \bag(R)$.
Let  $C_1, \ldots, C_h$ be the set of such components.  We observe that $\mvf(G, u, v) \ge \sum_{i=1}^h \mvf(G[C_i \cup \{u,v\}], u, v)$.
Because $\mvf(G[C_i \cup \{u,v\}], u, v) \ge 2\tw(G)+2$ for each $i$, this implies that the number of nodes added to $R$ in cases (2) and (3) is at most $\mvf(G)/(2 \tw(G) + 2)$.
The number of nodes added in cases (1) is at most the number of nodes added in cases (2) minus one.
In particular, whenever case (1) occurs, the number of nodes in $R$ that can be reached from the root without going through any other nodes in $R$ decreases by one, and in case (2), this number increases by one, and in case (3) this number is unchanged.
Therefore, $|R| \le \mvf(G)/(\tw(G) + 1) - 1$.

Now we have a set $S = \bag(R) \cup \{u,v\}$ so that $|S| \le \mvf(G)$ and for each connected component $C$ of $G \setminus S$ it holds that $\mvf(G[C \cup \{u,v\}], u, v) \le 3 \cdot \tw(G)+2$ and $|N(C)| \le 2\tw(G)+4$.
Because $|N(C)| \le 2\tw(G)+4$, we can explicitly compute the connected components $C_1, \ldots, C_h$ of $G \setminus S$ and their neighborhoods $N(C_1), \ldots, N(C_h)$ by $|V(G)|^{\Oh(\tw(G))}$ queries of size at most $2\tw(G)+4$.
Then, we use the algorithm of \Cref{the:twubmain:claimmincut} to compute a set $X_i$ for each $i \in [h]$ so that $X_i$ separates $u$ from $v$ in the graph $G[C_i \cup \{u,v\}] - uv$, and $|X_i| = \mvf(G[C_i \cup \{u,v\}], u, v)$.
By a similar argument as above, we conclude that $\sum_{i=1}^h |X_i| \le \mvf(G, u, v)$.
Finally, we observe that if there is no edge between $u$ and $v$, then $(S \setminus \{u,v\}) \cup \bigcup_{i=1}^h X_i$ separates $u$ from $v$ in $G$, so $((S \setminus \{u,v\}) \cup \bigcup_{i=1}^h X_i, u, v)$ is our final independence test that determines whether $uv \in E(G)$.
It has size at most $|S|+\sum_{i=1}^h |X_i| \le 2 \cdot \mvf(G)$.
\end{proof}

\section{CONCLUSION}
This paper establishes fundamental lower and upper bounds for the constraint-based learning of Markov networks. Our results expand upon existing work, which focused solely on the parameter maximum degree $\Delta(G)$~\citep{koller}, by identifying $\mvf(G)$ as the most critical parameter. Additionally, we explore learning parameterized by treewidth, a classical parameter in the context of probabilistic graphical models.

Regarding the sizes of conditioning sets, our findings are conclusive. We improve upon known bounds by demonstrating that tests of size $\mvf(G)$ are sufficient. (Recall that $\mvf(G) \leq \Delta(G)$.) We complement these results with a matching lower bound.



When considering the required number of tests, the situation becomes more intricate. Here, we hope that our work lays the foundation for future research in constraint-based learning of Markov networks from the perspective of structural graph theory. An intriguing open question remains: Can the upper bound presented in \Cref{the:twubmain} be considered optimal? Recall that the treewidth of a graph does not exceed its pathwidth. Thus by 
  \Cref{the:mainlb,the:xplowerbound} the bounds in \Cref{the:twubmain} on the number and sizes of tests in terms of treewidth are, in some sense, optimal. However, we do not exclude a possibility of replacing the treewidth with a graph parameter smaller than the treewidth. For example, could the treewidth in the statement of \Cref{the:twubmain} be substituted with the graph's degeneracy $\delta(G)$? It is worth noting that the degeneracy is always at most the treewidth.
In particular,  could the underlying graph $G$  be learned by $|V(G)|^{f(\delta(G))}$ independence tests of size at most $f(\mvf(G))$, for some function $f$?


\medskip\noindent\textbf{Acknowledgement.} The research leading to these results has been supported by the Research Council of Norway (grant no. 314528) and Trond Mohn forskningsstiftelse (grant no. TMS2023TMT01).

\clearpage

\bibliographystyle{plainnat}
\bibliography{bibliography}

\section*{Checklist}

 \begin{enumerate}

 \item For all models and algorithms presented, check if you include:
 \begin{enumerate}
   \item A clear description of the mathematical setting, assumptions, algorithm, and/or model. [Yes]
   \item An analysis of the properties and complexity (time, space, sample size) of any algorithm. [Yes]
   \item (Optional) Anonymized source code, with specification of all dependencies, including external libraries. [Not Applicable]
 \end{enumerate}

 \item For any theoretical claim, check if you include:
 \begin{enumerate}
   \item Statements of the full set of assumptions of all theoretical results. [Yes]
   \item Complete proofs of all theoretical results. [Yes]
   \item Clear explanations of any assumptions. [Yes]     
 \end{enumerate}

 \item For all figures and tables that present empirical results, check if you include:
 \begin{enumerate}
   \item The code, data, and instructions needed to reproduce the main experimental results (either in the supplemental material or as a URL). [Not Applicable]
   \item All the training details (e.g., data splits, hyperparameters, how they were chosen). [Not Applicable]
         \item A clear definition of the specific measure or statistics and error bars (e.g., with respect to the random seed after running experiments multiple times). [Not Applicable]
         \item A description of the computing infrastructure used. (e.g., type of GPUs, internal cluster, or cloud provider). [Not Applicable]
 \end{enumerate}

 \item If you are using existing assets (e.g., code, data, models) or curating/releasing new assets, check if you include:
 \begin{enumerate}
   \item Citations of the creator If your work uses existing assets. [Not Applicable]
   \item The license information of the assets, if applicable. [Not Applicable]
   \item New assets either in the supplemental material or as a URL, if applicable. [Not Applicable]
   \item Information about consent from data providers/curators. [Not Applicable]
   \item Discussion of sensible content if applicable, e.g., personally identifiable information or offensive content. [Not Applicable]
 \end{enumerate}

 \item If you used crowdsourcing or conducted research with human subjects, check if you include:
 \begin{enumerate}
   \item The full text of instructions given to participants and screenshots. [Not Applicable]
   \item Descriptions of potential participant risks, with links to Institutional Review Board (IRB) approvals if applicable. [Not Applicable]
   \item The estimated hourly wage paid to participants and the total amount spent on participant compensation. [Not Applicable]
 \end{enumerate}

 \end{enumerate}

\end{document}